\documentclass[10pt,twocolumn,letterpaper]{article}

\usepackage{iccv}
\usepackage{times}
\usepackage{epsfig}
\usepackage{graphicx}
\usepackage{amsmath}
\usepackage{amssymb}
\usepackage{amsthm}
\usepackage{bm}
\usepackage{booktabs}
\usepackage{multirow}
\usepackage{comment}

\newtheorem{theorem}{Theorem}

\newtheorem{proposition}{Proposition}

\newtheorem{lemma}{Lemma}

\newcommand{\keypoint}[1]{\vspace{0.1cm}\noindent\textbf{#1}\quad}

\usepackage[breaklinks=true,bookmarks=false]{hyperref}

\iccvfinalcopy 

\ificcvfinal\fi

\begin{document}

\title{Energy-Based Open-World Uncertainty Modeling \\ for Confidence Calibration}

\author{

Yezhen Wang$^{1}$\thanks{Equal contribution. Ordered by dice rolling. Correspondence to {yezhen.wang0305}@gmail.com.}
\quad
Bo Li$^{1}$\footnotemark[1]
\quad
Tong Che$^{2}$\footnotemark[1]
\quad
Kaiyang Zhou$^{3}$
\\
\quad
Ziwei Liu$^{3}$
\quad
Dongsheng Li$^{1}$
\\ \\
$^1$Microsoft Research Asia
\quad
$^{2}$MILA
\quad
$^{3}$S-Lab, Nanyang Technological University
}

\maketitle
\ificcvfinal\thispagestyle{empty}\fi

\begin{abstract}
Confidence calibration is of great importance to the reliability of decisions made by machine learning systems. However, discriminative classifiers based on deep neural networks are often criticized for producing overconfident predictions that fail to reflect the true correctness likelihood of classification accuracy. We argue that such an inability to model uncertainty is mainly caused by the closed-world nature in softmax: a model trained by the cross-entropy loss will be forced to classify input into one of $K$ pre-defined categories with high probability. To address this problem, we for the first time propose a novel $K$+1-way softmax formulation, which incorporates the modeling of open-world uncertainty as the extra dimension. To unify the learning of the original $K$-way classification task and the extra dimension that models uncertainty, we \textbf{1)} propose a novel energy-based objective function, and moreover, \textbf{2)} theoretically prove that optimizing such an objective essentially forces the extra dimension to capture the marginal data distribution. Extensive experiments show that our approach, Energy-based Open-World Softmax (EOW-Softmax), is superior to existing state-of-the-art methods in improving confidence calibration.
\end{abstract}
\section{Introduction}\label{sec:intro}
Given the considerable success achieved so far by deep neural networks (DNNs), one might be wondering if DNN-based systems can be readily deployed to solve real-world problems. On the one hand, DNNs can achieve high accuracy if trained with large-scale datasets~\cite{he2016deep}. But on the other hand, contemporary DNNs are often criticized for producing overconfident predictions~\cite{guo2017calibration}, which fail to represent the true correctness likelihood of accuracy. This has raised concerns over safety and reliability for using machine learning systems in real-world scenarios. Having a confidence-calibrated system is critical. For instance, in healthcare applications, the intelligence system should produce low-confidence predictions when it is uncertain about the input---say they differ significantly from the training data---so the decision-making process can be transferred to human doctors for more accurate diagnosis and safer handling.
Research on confidence calibration for DNNs has received increasing attention in recent years~\cite{guo2017calibration,liang2018enhancing,hsu2020generalized,lakshminarayanan2017simple,liang2018enhancing}. Since most classifiers are based on softmax, a common practice to improve calibration is to insert a temperature scaling parameter to the softmax function and adjust it in a validation set~\cite{guo2017calibration}. Besides, methods like Smoothing labels~\cite{szegedy2016rethinking,muller2019does}, which essentially combines the one-hot ground-truth vector with a uniform distribution, has also been shown effective in improving calibration. 

\begin{figure*}[!t]
    \centering
    \includegraphics[width=0.85\textwidth]{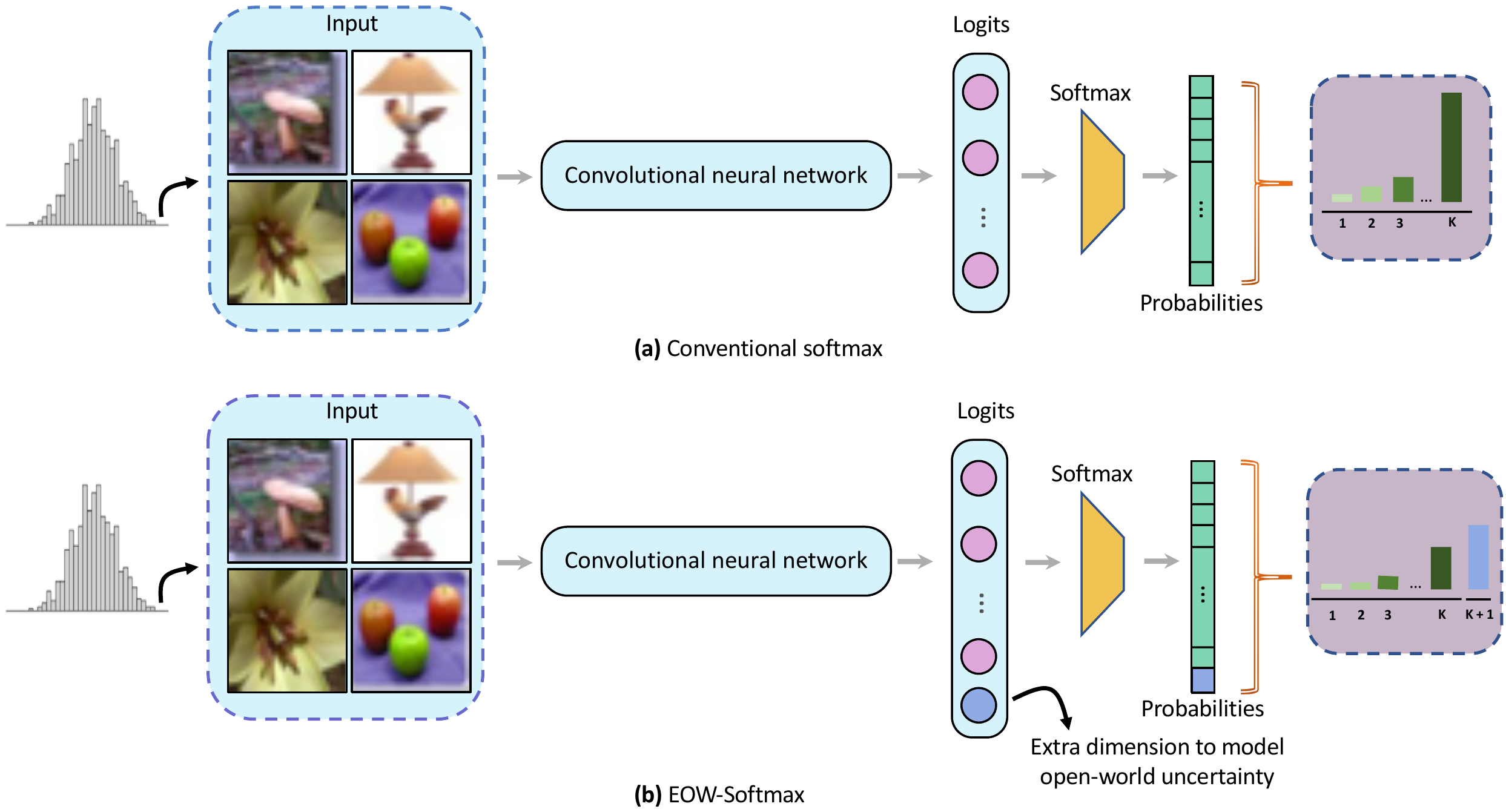}
    \caption{Comparison between (a) the conventional softmax and (b) our proposed Energy-based Open-World softmax (EOW-Softmax). Our new formulation introduces an extra dimension to model uncertainty, which is supposed to produce high scores when the input deviates from the training data distribution. In this way, the original $K$ classification scores can be well calibrated.}
    \label{fig:intro}
\end{figure*}

However, most existing confidence calibration methods have overlooked the underlying problem that causes neural network classifiers to generate overconfident predictions, i.e.~the inability to model \emph{uncertainty} in output probabilities. We argue that the culprit for causing such a problem is the closed-world nature in softmax~\cite{bendale2016towards,padhy2020revisiting}. This is easy to understand: during training the model is asked to classify input into one of $K$ pre-defined categories with high probability (due to the cross-entropy loss), and as such, the model has no choice but to assign one of the $K$ categories to any unseen data, likely with high probability as well.

A potential countermeasure is to adopt a $K+1$-way formulation where the new category can represent uncertainty about the input data. In this way, the $K$ classification scores might be better regularized, and hence better calibrated. However, learning such a classifier is challenging as we do not have access to those data with the $K+1$-th label, thus lacking supervision to teach the network when to give low/high confidence. 
Furthermore, designing the extra dimension is a non-trivial task as it is directly linked to the formulation of the learning objective. It is also unclear how such a dimension should be constructed, e.g., to design it as another logit produced by the same network or an independent branch that regresses to uncertainty.

In this paper, we propose \emph{Energy-based Open-World Softmax} (EOW-Softmax), a novel approach that introduces a $K+1$-way softmax based on energy functions~\cite{lecun2006tutorial}. Specifically, the neural network classifier is designed to produce $K+1$ logits, where the first $K$ dimensions encode the scores for the original $K$-way classification task, while the extra dimension aims to model open-world uncertainty. See Figure~\ref{fig:intro} for a comparison between a model based on the conventional softmax and that based on our EOW-Softmax. Besides, we resort to an energy-based $K+1$-way classification objective function to unify the learning of the $K$-way classification task and 
the uncertainty modeling. Further more, we theoretically justify that optimizing the proposed objective function essentially forces 
the summation of original $K$ softmax scores ($K+1$ scores in total) to be directly proportional to the marginal density $p(x)$, hence explaining why our EOW-Softmax helps calibrate a model's confidence estimates.

The \textbf{contributions} of this paper are summarized as follows. \textbf{1)} First, we overcome the closed-world softmax problem by transforming the conventional $K$-way softmax to a novel $K+1$-way formulation, where the extra dimension is designed to model open-world uncertainty. \textbf{2)} Second, a novel energy-based objective function is developed to unify the learning of the original $K$-way classification task and the uncertainty modeling. \textbf{3)} A theoretical proof is further provided to explain why our learning objective can help the network capture uncertainty. \textbf{4)} Finally, we conduct extensive experiments on standard benchmark datasets to demonstrate that our method can lead to a better calibrated model compared with other state-of-the-arts.

\section{Related Works}\label{sec:related_work}

\keypoint{Confidence Calibration}
With the emergence of deep learning technologies and their wide successes, concerns over whether they are reliable to be deployed in practice have also arisen. This is because researchers have found that contemporary deep neural networks (DNNs) often produce overconfident predictions~\cite{guo2017calibration}, even on input images that are totally unrecognizable to humans~\cite{NguyenYC14}. Many approaches for improving confidence calibration have been developed. A widely used method is temperature scaling~\cite{guo2017calibration,liang2018enhancing,hsu2020generalized}, which inserts a scaling parameter to the softmax formulation (called `temperature') and adjusts it in a validation set with a goal to `soften' the softmax probabilities. Regularization methods, such as label smoothing~\cite{szegedy2016rethinking} and Mixup~\cite{thulasidasan2019mixup}, have also been demonstrated effective in improving calibration. In particular, label smoothing modifies the ground-truth labels by fusing them with a uniform distribution, essentially forcing neural networks to produce `more flattened' probabilities; whereas Mixup is a data augmentation method that randomly mixes two instances at both the image and label space, with a byproduct effect of improving calibration. Bayesian methods have also been explored for calibration. For instance, Monte Carlo Dropout~\cite{gal2016dropout} applies dropout in both training and testing; Deep Ensembles~\cite{lakshminarayanan2017simple} uses as prediction the output averaged over an ensemble of models. Adding adversarial perturbations to the input has been found effective in smoothing the output probabilities~\cite{lakshminarayanan2017simple,liang2018enhancing}. In~\cite{lee2018training}, a GAN model~\cite{goodfellow2014generative} is trained to generate out-of-distribution (OOD) data and the classifier is encouraged to produce low-confidence probabilities on these data. Such an idea has also been investigated in~\cite{hein2019relu} where adversarial perturbations are utilized to synthesize OOD data. In~\cite{grathwohl2019your}, a Joint Energy-based Model (JEM) is proposed to improve calibration by learning the joint distribution based on energy functions~\cite{lecun2006tutorial}. A recent work~\cite{wald2021calibration} suggests that calibrating confidence across multiple domains is beneficial to OOD generalization~\cite{zhou2021domain}.
Studies on why neural networks produce overconfident predictions have also been covered in the literature. In~\cite{hein2019relu}, the authors suggest that ReLU neural networks are essentially piecewise linear functions, thus explaining why OOD data can easily cause softmax classifiers to generate highly confident output. In~\cite{nalisnick2019do}, the authors identify that data variance and model curvature cause most generative models to assign high density to OOD data. The authors in~\cite{padhy2020revisiting} point out that the overconfidence issue is related to the closed-world assumption in softmax, and design a distance-based one-vs-all (OvA) classifier as the countermeasure.

Two works related to ours are JEM~\cite{grathwohl2019your} and the OvA classifier~\cite{padhy2020revisiting}. Compared with JEM, our approach is much easier to train because we only need to optimize a \emph{single} classification objective to achieve both discriminative classifier learning and generative modeling (see Sec.~\ref{sec:method;subsec:theory}), while JEM has to simultaneously optimize two separate objectives. Moreover, JEM has ignored the closed-world softmax issue, which is addressed in this work with an augmented softmax. Compared with the OvA classifier, our approach is significantly different: we endow the classifier with the ability to model open-world uncertainty, which is attributed to the extra dimension in softmax learned via a novel energy-based objective function to capture the marginal data distribution; in contrast, the OvA classifier converts the $K$-way classification problem into multiple binary classification problems.

\keypoint{Energy-Based Models (EBMs)}
have been widely used in the area of generative modeling~\cite{zhao2017energy,gao2021learning,arbel2021generalized,che2020your}. The basic idea in EBMs is to learn dependencies between variables (e.g., images and labels) represented using energy functions; and to assign low energies to correct configurations while give high energies to incorrect ones~\cite{lecun2006tutorial}. However, training EBMs, especially on high-dimensional data like images, has been notoriously hard due to sampling issues~\cite{kim2016deep,grathwohl2021no}. A widely used sampler is Stochastic Gradient Langevin Dynamics (SGLD)~\cite{welling2011bayesian}, which injects noises to the parameter update and anneals the stepsize during the course of training. Following prior work~\cite{nijkamp2019learning,du2019implicit,grathwohl2019your}, we also leverage SGLD to optimize our energy-based objective function.


\begin{figure*}[t]
    \centering
    \includegraphics[width=0.98\textwidth]{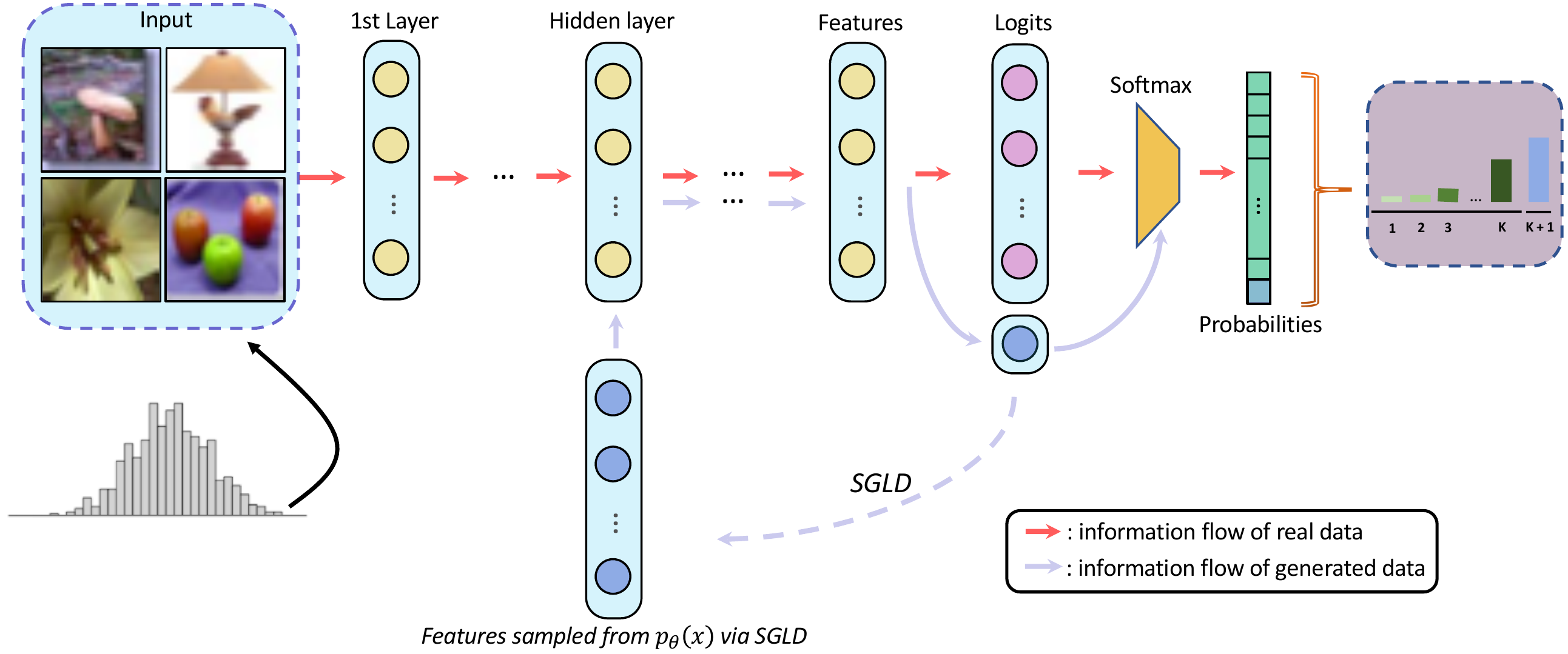}
    \caption{Model architecture of our approach \emph{EOW-Softmax}. The extra dimension introduced in the augmented softmax (dashed) is learned using an energy-based function to model open-world uncertainty, such that it can assign high uncertainty scores to abnormal input far away from the training data distribution, which in turn lower the classifier's confidence on the original $K$-way classification task. Note that the sampling in SGLD is performed on the latent (feature) space rather than the input (image) space.}
    \label{fig:ebm_framework}
\end{figure*}

\section{Methodology}\label{sec:method}
According to ~\cite{padhy2020revisiting}, we argue that the culprit for causing the overconfidence problem in most neural network classifiers' output is the closed-world nature in softmax. As a result, a model trained by the cross-entropy loss has to pick one of $K$ pre-defined categories with high confidence. To overcome this problem, we propose \emph{Energy-based Open-World Softmax}, or EOW-Softmax, to regularize the $K$-way classification scores in such a way that allows their confidence to be calibrated. The main idea in EOW-Softmax is to re-formulate the original $K$-way classification task as a novel $K$+1-way classification problem, where the extra dimension is designed to model \emph{open-world uncertainty}. To learn such a $K$+1-way classifier in an end-to-end manner, we propose a novel energy-based objective function, which essentially forces the extra dimension to be negatively correlated to the marginal data distribution. In this way, the $K$ classification scores are automatically calibrated to be less confident over input fallen beyond the training data distribution. See Figure~\ref{fig:ebm_framework} for an overview of our model architecture. The rest of this section are organized as follows. In Sec.~\ref{sec:method;subsec:bg_ebms}, we provide a brief background on energy-based models (EBMs), which are required by our approach to construct the objective function. In Sec.~\ref{sec:method;subsec:eow_softmax}, we discuss in detail the design of EOW-Softmax. Sec.~\ref{sec:method;subsec:theory} gives a theoretical insight on why EOW-Softmax can help calibrate a model's confidence estimates.

\subsection{A Brief Background on EBMs}\label{sec:method;subsec:bg_ebms}
The main building block in EBMs is an energy function $E_{\theta}: \mathbb{R}^D \to \mathbb{R}$ (parameterized by $\theta$), which aims to map a $D$-dimensional datapoint to a scalar.\footnote{When the input is an image, $D$ can be understood as the length of the flattened tensor.} The learning is designed in such a way that $E_{\theta}$ can assign low energies to observed configurations of variables while give high energies to unobserved ones~\cite{lecun2006tutorial}. With $E_{\theta}$, any probability density $p(x)$ for $x \in \mathbb{R}^D$ in an EBM can be written as
\begin{equation} \label{eq:ebm_density}
p_{\theta}(x) = \frac{\exp(- E_{\theta}(x))}{Z(\theta)},
\end{equation}
where $Z(\theta) = \int_x \exp(- E_{\theta}(x))$ denotes the normalizing constant, also known as the partition function. An EBM can be represented by using any function as long as the function can generate a single scalar given some input $x$. In this work, we assume $E_{\theta}$ is represented by a deep neural network. People usually adopt gradient estimation to optimize an EBMs~\cite{kim2016deep,grathwohl2021no} and sample data from it by Markov Chain Monte Carlo (MCMC) methods~\cite{geyer1991markov, geyer1992practical, hinton2002training, welling2011bayesian}.


\subsection{Energy-Based Open-World Softmax}\label{sec:method;subsec:eow_softmax}

\keypoint{Open-World Softmax}
As discussed before, conventional softmax-based classifiers lack the ability to model open-world uncertainty. To address this problem, we design a neural network classifier to output probabilities on $K+1$ categories (see Figure~\ref{fig:ebm_framework}), with the $K$+1-th score representing open-world uncertainty---the network should be able to produce high uncertainty scores to abnormal input, which in turn can lower the confidence on the original $K$ categories' prediction. Let $f_\theta: \mathbb{R}^D \to \mathbb{R}^{K+1}$ be our neural network model (excluding the softmax layer), which produces $K+1$ logits, and $f_\theta(x)[i]$ the $i$-th logit given input $x$, with $i \in \{ 1, ..., K, K+1 \}$. The output probabilities can then be obtained by passing these $K+1$ logits to a softmax normalization layer, i.e.
\begin{equation}
h_\theta(x)[i] = \frac{\exp(f_\theta(x)[i])}{\sum_{j=1}^{K+1} \exp(f_\theta(x)[j])},
\end{equation}
where $h_\theta$ is the combination of the neural network $f_\theta$ and the softmax normalization layer.

\keypoint{Energy-Based Learning Objective}
Now the question is how to design a learning objective that allows $h_\theta(x)[K+1]$ to encode uncertainty? Our idea here is to associate the score of $h_\theta(x)[K+1]$ to the marginal data distribution. Intuitively, when the input comes from within the training data distribution $p(x)$, the model is supposed to be confident in its decision, and therefore, $h_\theta(x)[K+1]$ should be low (conversely, $\sum_{i=1}^K h_\theta(x)[i]$ should be high). If the input deviates from the training data distribution, the model should become uncertain about whether its decision is correct. In this case, $h_\theta(x)[K+1]$ should be high to indicate a higher level of uncertainty, which naturally forces $\sum_{i=1}^K h_\theta(x)[i]$ to stay low (due to the softmax normalization). However, directly training $h_\theta(x)[K+1]$ to capture the marginal distribution $p(x)$ (i.e.~generative modeling) is difficult~\cite{nalisnick2019do}. Instead, we propose a novel learning objective with the help of EBMs~\cite{lecun2006tutorial}. First, we define our energy function as
\begin{equation}\label{eq:our_energy}
E_\theta(x) = \log h_\theta(x)[K+1].
\end{equation}
Then, our energy-based objective function is defined as
\begin{equation}\label{eq:our_obj}
\min\limits_{\theta}
\mathbb{E}_{p(x)} \bigg[- \log h_\theta(x)[y] \bigg] + \lambda \mathbb{E}_{p_{\bar{\theta}}(x)} \bigg[- \log h_\theta(x)[K+1] \bigg],
\end{equation}
where $\lambda > 0$ is a hyper-parameter; the first term is the maximum log-likelihood objective for the $K$-way classification task using the ground-truth label $y$; the second term can also be seen as maximum log-likelihood objective---for recognizing data sampled from $p_{\bar{\theta}}(x)$. Note that $p_{\bar{\theta}}(x)$ denotes the model distribution with frozen parameters of $p_{\theta}(x)$ of current iteration ($p_{\bar{\theta}}(x)$ will always be the same as $p_{\theta}(x)$ but without the gradient calculation of $\theta$ since parameter  $\theta$ is frozen here and $\bar{\theta}$ should be regarded as a constant in each updates). We will show later in Sec.~\ref{sec:method;subsec:theory} that optimizing Eq.~\eqref{eq:our_obj} can actually lead the summation of rest $K$ softmax scores of original classes to be directly proportional to the marginal density $p(x)$, which in turn can make the $K+1$-th softmax score be negatively correlated to $p(x)$.




\keypoint{SGLD-Based Optimization}
We approximate the expectation in the second term in Eq.~\eqref{eq:our_obj} using a sampler based on Stochastic Gradient Langevin Dynamics (SGLD)~\cite{welling2011bayesian}. Specifically, the SGLD sampling process follows
\begin{equation}\label{eq:sgld}
z_{t+1} = z_t - \frac{\alpha}{2} \frac{\partial E_\theta(z_t)}{\partial z_t} + \sqrt{\alpha} \epsilon, \quad \epsilon \sim \mathcal{N}(0, I),
\end{equation}
where $t$ denotes the SGLD iteration, $\alpha$ the step-size, and $\epsilon$ a random noise sampled from a normal distribution. In practice, $\alpha$ is usually fixed as a constant.

Most SGLD-based methods draw samples from the image space. This forces the information to flow through the entire neural network, which is computationally expensive. Inspired by~\cite{che2020your}, we choose to draw samples from the latent space (see Figure~\ref{fig:ebm_framework}). Therefore, $z$ in Eq.~\eqref{eq:sgld} represents features rather than an image. Such a design significantly accelerates the training since the information only goes partially through the network model, which allows much deeper architectures such as ResNet-101~\cite{he2016deep} to fit into limited resources. Moreover, the latent space is typically smoother than the image space~\cite{bengio2013}, which facilitates the estimate of gradients. 

\subsection{Theoretical Insight}\label{sec:method;subsec:theory}
In order to prove that our objective can force $h_\theta(x)[K+1]$ to be negatively correlated with $p(x)$, i.e.~representing uncertainty, we tend to show that in theory, optimizing our objective in Eq.~\eqref{eq:our_obj} is equivalent to minimize the KL divergence between $p(x)$ and another EBM-modeled distribution $q_\theta(x)$\footnote{It is worth noting that this $q_\theta(x)$ is not the modeled EBMs $p_\theta(x)$ in Sec.~\ref{sec:method;subsec:eow_softmax}.}, where $q_\theta(x)$ is defined by energy function
\begin{equation}\label{eq:energy_prime}
E'_\theta(x) = -\log \sum\limits_{i=1}^{K} h_\theta(x)[i].
\end{equation}

To this end, we introduce an extra objective
\begin{equation}\label{eq:our_obj_v2}
\min\limits_{\theta}
\mathbb{E}_{p(x)} \bigg[- \log h_\theta(x)[y] \bigg] + \mathbb{E}_{q_{\bar{\theta}}(x)} \bigg[ \log \sum\limits_{i=1}^{K} h_\theta(x)[i] \bigg],
\end{equation}

%
Similar to Eq.~\eqref{eq:our_obj}, here $\bar{\theta}$ means the parameters are frozen. We will show that optimizing Eq.~\eqref{eq:our_obj} essentially optimize Eq.~\eqref{eq:our_obj_v2} and optimizing Eq.~\eqref{eq:our_obj_v2} is an equivalent of $\min D_{KL}(p(x) || q_\theta(x))$ in following Proposition~\ref{prop:equivalet_obj_objv2} and Theorem~\ref{thm:main}, respectively.

\begin{proposition}
\label{prop:equivalet_obj_objv2}
Given two EBMs $p_{\bar{\theta}}(x)$ and $q_{\bar{\theta}}(x)$ with energy functions defined in Eqs.~\eqref{eq:our_energy} and ~\eqref{eq:energy_prime} respectively, the optimization of Eq.~\eqref{eq:our_obj} is actually equivalent to optimize a combination of one $K$-way classification objective and Eq.~\eqref{eq:our_obj_v2} with some suitable coefficient $\mu$.
\end{proposition}
\begin{proof}
Since the first optimization term in Eq.~\eqref{eq:our_obj_v2} is identical with Eq.~\eqref{eq:our_obj} as well as the objective of maximum log-likelihood of $K$-way classification problems, we only need to consider the second terms in both equations and prove that they are equivalent to each other. Specifically, for the gradient of the second term of Eq.~\eqref{eq:our_obj_v2}, we have\footnote{The equality between Eqs.~\eqref{eq:prop_2} and \eqref{eq:prop_3} holds because $q_{\bar{\theta}}(x) = q_{\theta}(x) = \frac{\sum\limits_{i=1}^{K} h_\theta(x)[i]}{Z'(\theta)}$ and $\frac{\partial}{\partial \theta} \log \sum\limits_{i=1}^{K} h_\theta(x)[i] = \frac{\frac{\partial}{\partial \theta} \sum\limits_{i=1}^{K} h_\theta(x)[i] }{\sum\limits_{i=1}^{K} h_\theta(x)[i]}$.}

\begin{align}
\label{eq:prop_1}
& \frac{\partial}{\partial \theta} \ \mathbb{E}_{ q_{\bar{\theta}}(x)} \Big[ \log \sum\limits_{i=1}^{K} h_\theta(x)[i] \Big] \\
= \ & \int_x q_{\bar{\theta}}(x) \frac{\partial}{\partial \theta} \log \sum\limits_{i=1}^{K} h_\theta(x)[i] \label{eq:prop_2} \\
= \ & \int_x \frac{\sum\limits_{i=1}^{K} h_\theta(x)[i]}{Z'(\theta)} \cdot \frac{\frac{\partial}{\partial \theta} \sum\limits_{i=1}^{K} h_\theta(x)[i] }{\sum\limits_{i=1}^{K} h_\theta(x)[i]} \label{eq:prop_3} \\
= \ & \frac{1}{Z'(\theta)} \int_x \frac{\partial}{\partial \theta} \sum\limits_{i=1}^{K} h_\theta(x)[i] \label{eq:prop_4} \\
= \ & -\frac{1}{Z'(\theta)} \int_x \frac{\partial}{\partial \theta} h_\theta(x)[K+1] \label{eq:prop_5} \\
= \ & -\frac{Z(\theta)}{Z'(\theta)} \int_x \frac{h_\theta(x)[K+1]}{Z(\theta)} \cdot \frac{\frac{\partial}{\partial \theta} h_\theta(x)[K+1]}{h_\theta(x)[K+1]} \label{eq:prop_6} \\
= \ & -\frac{Z(\theta)}{Z'(\theta)} \int_x p_{\bar{\theta}}(x)\frac{\partial}{\partial \theta} \log h_\theta(x)[K+1], \label{eq:prop_7}
\end{align}
where $Z'(\theta)$ and $Z(\theta)$ represents the partition functions of $q_\theta(x)$ and $p_\theta(x)$ respectively. 
If we use $\mu$ to denote $\frac{Z(\theta)}{Z'(\theta)}$, we can restate the Eq.~\eqref{eq:prop_7} as
$$
\mu \frac{\partial}{\partial \theta} \mathbb{E}_{p_{\bar{\theta}}(x)} \bigg[- \log h_\theta(x)[K+1] \bigg],
$$
which is exactly the gradient of Eq.~\eqref{eq:our_obj}'s second term. As a result, the objective of Eq.~\eqref{eq:our_obj} and of Eq.~\eqref{eq:our_obj_v2} are same under a suitable coefficient $\mu$ in each iteration. Further more, remembering that the first term of Eq.~\eqref{eq:our_obj} is an equivalent of the objective of maximum log-likelihood of $K$-way classification problems, we consequently conclude that optimizing our objective of Eq.~\eqref{eq:our_obj} essentially optimize Eq.~\eqref{eq:our_obj_v2} and a $K$-way classification objective.
\end{proof}

According to Proposition~\ref{prop:equivalet_obj_objv2}, we know that our learning objective of Eq.~\eqref{eq:our_obj} can optimize the discriminative $K$-way classification objective and the generative modeling objective in a unified discriminative objective, which has never been explored in existing confidence calibration work. 

Moreover, if we can further prove that optimizing Eq.~\eqref{eq:our_obj_v2} is \textit{de facto} an equivalent of minimizing the KL-divergence between $p(x)$ and the distribution $q_\theta(x)$, then our objective of Eq.~\eqref{eq:our_obj} could minimize this KL-divergence either due to Proposition~\ref{prop:equivalet_obj_objv2}. To prove that, we need to recur to Lemma \ref{lem:kl_grad_ebm} based on~\cite{kim2016deep}, which shows how to efficiently compute the gradient of the KL divergence between the real distribution and an approximated distribution modeled via EBMs.
\begin{lemma}\label{lem:kl_grad_ebm}
\vspace{-2mm}
Given a training dataset with data $\{x\}$ sampled from distribution $r(x)$, and an energy model distribution $r_\phi(x)$ parameterized by $\phi$ and associated to an energy function $E_\phi(x)$, the objective of minimizing $D_{KL}(r(x) || r_\phi(x))$ can be optimized by descending the following gradient w.r.t $\phi$,
\begin{equation}\label{eq:kl_grad_ebm}
\mathop{\mathbb{E}}\limits_{x^+ \sim r(x)} \Bigg[ \frac{\partial E_\phi(x^+)}{\partial \phi} \Bigg] - \mathop{\mathbb{E}}\limits_{x^- \sim r_\phi(x)} \Bigg[\frac{\partial E_\phi(x^-)}{\partial \phi}\Bigg].
\end{equation}
\end{lemma}
The proof can refer ~\cite{kim2016deep}. By descending the gradient in Eq.~\eqref{eq:kl_grad_ebm}, the first term decreases the energy of samples $x^+$ drawn from the data distribution, while the second term increases the energy of samples $x^-$ drawn from the energy model distribution. Based on Lemma~\ref{lem:kl_grad_ebm}, we can introduce following theorem followed by a proof.
\begin{theorem}
\label{thm:main}
Let $p(x)$ denote the training data distribution, and $q_\theta(x)$ the energy model distribution represented by the energy function $E'_\theta(x)$ defined in Eq.~\eqref{eq:energy_prime}, we can achieve minimization of $D_{KL}(p(x) || q_\theta(x))$ by optimizing Eq.~\eqref{eq:our_obj_v2}.
\end{theorem}
\begin{proof}
According to Lemma~\ref{lem:kl_grad_ebm}, the KL divergence between $p(x)$ and $q_\theta(x)$ can be optimized by descending the gradient in Eq.~\eqref{eq:kl_grad_ebm}. We can replace the need of computing expectation over the parameterized density $q_\theta(x)$ in Eq.~\eqref{eq:kl_grad_ebm} by fixing the parameters $\theta$, denoted by $q_{\bar{\theta}}(x)$ \footnote{only the numerical value of Eq.~\eqref{eq:kl_grad_ebm} matters in optimization process.}. As such, Eq.~\eqref{eq:kl_grad_ebm} is converted to
\begin{equation}\label{eq:kl_grad_ebm_fixedparam}
\mathbb{E}_{p(x)} \Bigg[ \frac{\partial E'_\theta(x)}{\partial \theta} \Bigg] - \mathbb{E}_{q_{\bar{\theta}}(x)} \Bigg[ \frac{\partial E'_\theta(x)}{\partial \theta} \Bigg].
\end{equation}
Now we can optimize $D_{KL}(p(x) || q_\theta(x))$ via an objective $\mathbb{E}_{p(x)} \Big[ E'_\theta(x) \Big] - \mathbb{E}_{q_{\bar{\theta}}(x)} \Big[ E'_\theta(x) \Big]$
which holds a numerically same gradient with regard to $\theta$ as Eqs.~\eqref{eq:kl_grad_ebm_fixedparam} and ~\eqref{eq:kl_grad_ebm}. Then we have\footnote{the inequality holds because of the fact $\forall y \leq K, \ \ \sum\limits_{i=1}^{K} h_\theta(x)[i] \geq h_\theta(x)[y].$}
\begin{equation}
\label{eq:final}
\begin{aligned}
&\mathbb{E}_{p(x)} \Bigg[ E'_\theta(x) \Bigg] - \mathbb{E}_{q_{\bar{\theta}}(x)} \Bigg[ E'_\theta(x) \Bigg] \\
= \ &\mathbb{E}_{p(x)} \Bigg[ -\log \sum\limits_{i=1}^{K} h_\theta(x)[i] \Bigg] + \mathbb{E}_{q_{\bar{\theta}}(x)} \Bigg[ \log \sum\limits_{i=1}^{K} h_\theta(x)[i] \Bigg] \\
\leq \ &\mathbb{E}_{p(x)} \Bigg[ -\log h_\theta(x)[y] \Bigg] + \mathbb{E}_{q_{\bar{\theta}}(x)} \Bigg[ \log \sum\limits_{i=1}^{K}h_\theta(x)[i] \Bigg].
\end{aligned}
\end{equation}
Therefore, Eq.~\eqref{eq:our_obj_v2} is an upper bound of an equivalent variant of the KL-divergence between $p(x)$ and $q_\theta(x)$. 
\end{proof}
Combing Proposition~\ref{prop:equivalet_obj_objv2} and Theorem~\ref{thm:main}, we can conclude that our objective of Eq.~\eqref{eq:our_obj} can minimize $D_{KL}(p(x) || q_\theta(x))$. Therefore, once our objective converged, we obtain $p(x) \simeq q_\theta(x) = \frac{\exp(-E'_\theta(x))}{Z'(\theta)} \propto \sum\limits_{i=1}^{K} h_\theta(x)[i]$, which result in the summation of $K$ softmax scores of original classes to be directly proportional to the marginal density $p(x)$ and in turn make the $K+1$-th softmax score be negatively correlated to $p(x)$.


\begin{table*}[t]
\centering\small 
\caption{Comparison between our approach EOW-Softmax and nine baselines on four benchmark datasets. It is clear that our approach generally leads to a better calibrated model than the baselines (lower ECE \& NLL), while maintaining the accuracy. $\uparrow$: the higher the better. $\downarrow$: the lower the better.}
\label{tab:4_main_benchmarks}
\resizebox{\linewidth}{!}{%
\begin{tabular}{c |ccc| ccc |ccc| ccc}
\toprule
\multirow{2}{*}{Method} & \multicolumn{3}{c|}{\centering MNIST (MLP)} & \multicolumn{3}{c|}{\centering CIFAR10 (VGG11)} & \multicolumn{3}{c|}{\centering CIFAR100 (ResNet50)} & \multicolumn{3}{c}{\centering Tiny-ImageNet (ResNet50)} \\
\cline{2-13}
 &  Acc\% $\uparrow$ &  ECE\% $\downarrow$ & NLL $\downarrow$ &  Acc\% $\uparrow$ &  ECE\% $\downarrow$ & NLL $\downarrow$&  Acc\% $\uparrow$ &  ECE\% $\downarrow$ & NLL $\downarrow$&  Acc\% $\uparrow$ &  ECE\% $\downarrow$ & NLL $\downarrow$ \\
\hline
Vanilla Training & 98.32 & 1.73 & 0.29 & 90.48 & 6.30 & 0.43 & 71.57 & 19.1 & 1.58 & 46.71 & 25.2 & 2.95 \\
TrustScore~\cite{TrustScore} & 98.32 & 2.14 & 0.26 & 90.48 & 5.30 & 0.40 & 71.57 & 10.9 & 1.43 & 46.71 & 19.2 & 2.75 \\
MC-Dropout~\cite{gal2016dropout} & 98.32 & 1.71 & 0.34 & 90.48 & 3.90 & 0.47 & 71.57 & 9.70 & 1.48 & 46.72 & 17.4 & 3.17 \\
Label Smoothing~\cite{szegedy2016rethinking} & 98.77 & 1.68 & 0.30 & 90.71 & 2.70 & 0.38 & 71.92 & 3.30 & 1.39 & \textbf{47.19} & 5.60 & 2.93 \\ 
Mixup~\cite{thulasidasan2019mixup} & 98.83 & 1.74 & 0.24 & 90.59 & 3.30 & 0.37 & \textbf{71.85} & 2.90 & 1.44 & 46.89 & 6.80 & 2.66 \\
JEM~\cite{grathwohl2019your} & 97.23 & 1.56 & 0.21 & 90.36 & 3.30 & 0.34 & 70.28 & 2.46 & 1.31 & 45.97 & 5.42 & 2.47 \\ 
OvA DM~\cite{padhy2020revisiting} & 96.67 & 1.78 & 0.27 & 89.56 & 3.55 & 0.37 & 70.11 & 3.58 & 1.40 & 45.55 & 4.22 & 2.50 \\
Temperature Scaling~\cite{guo2017calibration} & 95.14 & 1.32 & 0.17 & 89.83 & 3.10 & 0.33 & 69.84 & 2.50 & 1.23 & 45.03 & 4.80 & 2.59 \\
DBLE~\cite{Xing2020Distance-Based} & 98.69 & 0.97 & \textbf{0.12} & \textbf{90.92} & \textbf{1.50} & 0.29 & 71.03 & 1.10 & 1.09 & 46.45 & 3.60 & 2.38 \\
\hline
{EOW-Softmax} (\emph{ours}) & \textbf{98.91} & \textbf{0.88} & 0.15 & 90.24 & 1.57 & \textbf{0.25} & 71.33 & \textbf{1.08} & \textbf{1.03} & 46.97 & \textbf{3.45} & \textbf{2.22}  \\

\bottomrule
\end{tabular}
}
\end{table*}

\begin{figure*}[t]
    \centering
    \includegraphics[width=1\linewidth]{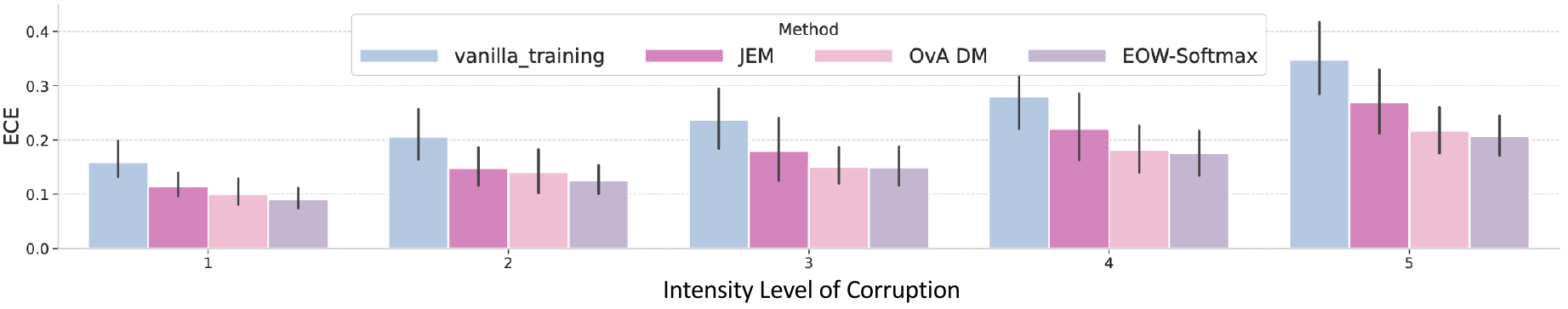}
    \caption{Results on CIFAR-100-C. Each bar represents the mean ECE on 19 different corruptions, with the vertical line segment denoting the standard deviation. In general, EOW-Softmax achieves the lowest ECE under five different corruption levels.}
    \label{fig:barplot}
\end{figure*}

\section{Experiments}
\subsection{Experimental Setup}

\keypoint{Settings}
We adopt three settings to evaluate our approach. 1) \emph{Confidence Calibration}: This aims to evaluate the effectiveness of a method in improving confidence calibration. Following~\cite{Xing2020Distance-Based}, we use four datasets: MNIST~\cite{lecun1998gradient} (MLP), CIFAR-10~\cite{krizhevsky2009learning} (VGG11~\cite{simonyan2014very}), CIFAR-100~\cite{krizhevsky2009learning} (ResNet50~\cite{he2016deep}), and Tiny-ImageNet~\cite{deng2009imagenet} (ResNet50). 
Network architectures used for these datasets are indicated in the parentheses. 2) \emph{OOD Detection}: A ResNet50 classifier is trained on CIFAR-100 and tested on the combination of CIFAR-100's test split and an OOD dataset, i.e.~CIFAR-10 and SVHN~\cite{netzer2011reading}. The goal for the classifier is to assign low confidence to as many OOD samples as possible. The accuracy is computed on predictions with confidence higher than a threshold. 3) \emph{Robustness under Corruption}: A classifier is trained on CIFAR-100. Its calibration performance is evaluated on CIFAR-100-C~\cite{hendrycks2019robustness}, where the images are perturbed by 19 different corruptions with five intensity levels.

\keypoint{Baselines}
We compare our approach with nine baseline methods: MC-Dropout~\cite{gal2016dropout}, Temperature Scaling~\cite{guo2017calibration}, Mixup~\cite{thulasidasan2019mixup}, Label Smoothing~\cite{szegedy2016rethinking}, TrustScore~\cite{TrustScore}, JEM~\cite{grathwohl2019your}, DBLE~\cite{Xing2020Distance-Based}, and OvA DM~\cite{padhy2020revisiting}.


\keypoint{Evaluation Metrics}
Following prior work~\cite{guo2017calibration}, we use two metrics to assess how well the confidence in a model's predictions is calibrated, namely Expected Calibration Error (ECE)~\cite{naeini2015obtaining} and Negative Log-Likelihood (NLL). The lower the ECE/NLL, the better the calibration performance. Below we explain in detail how these two metrics are calculated. ECE approximates the expectation of the difference between accuracy and confidence (i.e.~the likelihood on the predicted label $\hat{y}$), which can reflect how well the prediction confidence aligns with the true accuracy. Specifically, the confidence estimates made on all test samples are partitioned into $L$ equally spaced bins (following~\cite{guo2017calibration,Xing2020Distance-Based,padhy2020revisiting}, $L=15$), and the difference between the average confidence and accuracy within each bin $I_l$ is calculated,
\begin{equation}
\begin{aligned}
\text{ECE} = \sum\limits_{l=1}^{L} \frac{1}{N} |\sum\limits_{x \in I_l} p(\hat{y}|x) - \sum\limits_{x \in I_l} 1(\hat{y}=y)|,
\end{aligned}    
\end{equation}
where $N$ denotes the total number of samples in the test set. NLL computes the average negative log-likelihood on all test samples,
\begin{equation}
\begin{aligned}
\text{NLL} = - \frac{1}{N} \sum\limits_{m=1}^N \log p(\hat{y}_m|x_m).
\end{aligned}    
\end{equation}

\keypoint{Implementation Details}
We use the SGD optimizer with the learning rate of 1e-4, the momentum of 0.9, and the weight decay of 5e-4. The batch size is set to 64. The number of epochs is 200. The learning rate is decayed by 0.1 at the 100-th and 150-th epoch, respectively. For SGLD, we use a constant step size of 2 and a standard deviation of 1e-3 (see Eq~\eqref{eq:sgld}). The number of updates in each SGLD round is set to 100. To ensure that the results are convincing, we run each experiment 5 times with different random seeds and average their results. 

\subsection{Main Results}

\keypoint{Confidence Calibration}
We first evaluate the calibration performance on four standard datasets, namely MNIST, CIFAR-10/100 and Tiny-ImageNet. The results are shown in Table~\ref{tab:4_main_benchmarks}. In general, our approach EOW-Softmax obtains the best overall calibration performance on most datasets. Comparing with Vanilla Training, we observe that our EOW-Softmax achieves a similar test accuracy, while significantly improves the calibration performance in terms of ECE, especially on the three challenging datasets with natural images---6.30\%$\to$1.57\% on CIFAR-10, 19.1\%$\to$1.08\% on CIFAR-100, and 25.2\%$\to$3.45\% on Tiny-ImageNet. These results strongly suggest that our energy-based open-world uncertainty modeling has great potential for real-world applications as it shows a good balance between test accuracy and calibration performance. JEM and OvA DM are two related methods to ours. JEM is based on joint distribution modeling while OvA DM transforms the conventional softmax classifier to a distance-based one-vs-all classifier. The comparison with these two methods shows that our approach is clearly better in all metrics, which advocates our design of the $K$+1-way softmax for modeling open-world uncertainty. Compared with the top-performing baselines, i.e.~Temperature Scaling and DBLE, our EOW-Softmax is highly competitive---it obtains the best ECE on all datasets except CIFAR-10 where the performance is only slightly worse than DBLE. Among these three methods, EOW-Softmax performs the best in maintaining the original test accuracy, whereas Temperature Scaling has to sacrifice the test accuracy in exchange for improvement on ECE. This is because for fair comparison (all methods only have access to the training set), Temperature Scaling has to separate a validation set out from the original training set for tuning its scaling parameter, which reduces the amount of training data.

\begin{table}[t]
\setlength{\tabcolsep}{4pt}
\centering\small 
\caption{Test accuracy on the combination of in-distribution and OOD test set using ResNet50 trained on CIFAR-100.}
\begin{tabular}{ c | c | cccc}
\toprule
 \multirow{2}{*}{OOD dataset} & \multirow{2}{*}{Method} & \multicolumn{4}{c}{\centering Probability threshold} \\
\cline{3-6}
 &   & 0 & .25 & .5 & .75 \\
\hline
\multirow{4}{*}{CIFAR10} & Vanilla Training & 37.24 & 42.31 & 47.67 & 58.90 \\
 & JEM~\cite{grathwohl2019your} & \textbf{42.55} & 46.88 & 0.53 & 63.21 \\
 & OvA DM~\cite{padhy2020revisiting} & 39.78 & 48.54 & 54.31 & 65.20 \\
 & EOW-Softmax & 37.65 & \textbf{50.11} & \textbf{57.32} & \textbf{69.00} \\
\hline
\multirow{4}{*}{SVHN} & Vanilla Training & 20.24 & 21.33 & 24.38 & 26.90 \\
 & JEM~\cite{grathwohl2019your} & 19.87 & 22.57 & 26.22 & 30.75 \\
 & OvA DM~\cite{padhy2020revisiting} & 20.08 & 23.99 & 26.08 & 30.32 \\
 & EOW-Softmax & \textbf{20.21} & \textbf{25.58} & \textbf{28.13} & \textbf{32.68} \\ 
\bottomrule
\end{tabular}
\label{tab:ablation_OOD}
\end{table}

\keypoint{OOD Detection}
We follow~\cite{padhy2020revisiting} to simulate real-world scenarios by training a classifier on CIFAR-100 and testing on the combination of CIFAR-100's test set and an OOD dataset (CIFAR-10/SVHN). The classifier is required to assign low confidence to OOD samples such that their predictions can be rejected by a pre-defined threshold. The results are reported in Table~\ref{tab:ablation_OOD} where we compare our approach with JEM and OvA DM (as these two methods are most related to ours), as well as the vanilla training method. The probability threshold is linearly increased from 0 to 0.75. Only predictions with confidence higher than this threshold are kept. From the results, we observe that EOW-Softmax outperforms all baselines with clear margins in different thresholds (except 0). This indicates that a classifier trained with EOW-Softmax suffers much less from the overconfidence problem than the competitors, and is thus safer to be deployed in practical applications.

\keypoint{Robustness under Corruption}
We also evaluate our approach on corrupted images, i.e.~CIFAR-100-C, and compare with JEM and OVA DM. Specifically, there are five different intensities of corruption as defined in~\cite{hendrycks2019robustness}, each with 19 different corruption types. Under each intensity level of corruption, we test a model's ECE on images of all corruption types, and report their average result and the standard deviation. The comparison is illustrated in Figure~\ref{fig:barplot}. Overall, EOW-Softmax performs favorably against JEM and OVA DM, as well as the vanilla training baseline. Though both EOW-Softmax and JEM are based on energy models for generative modeling, JEM clearly has a larger variation in performance among different corruption types. This is because JEM is more difficult to train---it optimizes the log-likelihood of distribution independently from the classifier learning. In contrast, the `generative modeling' in EOW-Softmax is seamlessly integrated into the $K$+1-way classification task, which has a better training stability.

\subsection{Ablation study}
\keypoint{Hyper-parameter}
Recall that $\lambda$ in Eq.~\eqref{eq:our_obj} balances between the standard $K$-way classification loss and the energy-based loss for uncertainty modeling. We experiment with different values (1, 0.1, and 0.01) on CIFAR-10/100 to see the impact of this hyper-parameter. The results in Table~\ref{tab:ablation_lambda} show that $\lambda=0.1$ leads to the best calibration performance (lowest ECE values) on both datasets.


\keypoint{Where to Apply the SGLD Sampling?}
In our approach, SGLD sampling is applied to the latent feature space rather than the pixel space as in most EBMs. To justify this design, we experiment on CIFAR-100 with different variants of our approach where the SGLD sampling is applied to different positions, including the pixel space, stage-1, stage-2, and stage-3 (in the neural network). In addition to the test accuracy and the ECE, we also report the training speed (second per iteration) measured using a Tesla K80 GPU. Table~\ref{tab:ablation_lambda} shows that applying the SGLD sampling to the pixel space incurs huge computation overhead, while shifting the sampling to the latent space significantly improves the training speed without sacrificing the calibration performance.


\begin{table}[t]
\centering\small 
\caption{Ablation study on the impact of $\lambda$ in Eq.~\eqref{eq:our_obj}.}
\label{tab:ablation_lambda}
\begin{tabular}{c | c |ccc}
\toprule
\multirow{2}{*}{Model} &\multirow{2}{*}{$\lambda$} & \multicolumn{3}{c}{\centering CIFAR10} \\
\cline{3-5}
& &  Acc\% $\uparrow$ &  ECE\% $\downarrow$ & NLL $\downarrow$ \\
\hline
\multirow{3}{*}{VGG11} & 1 & 89.81 & 2.11 & 0.28 \\
& 0.1 &  \textbf{90.24} & \textbf{1.57} & \textbf{0.25} \\
& 0.01 &  90.11 & 3.37 & 0.36 \\
\hline
\hline
\multirow{2}{*}{Model} &\multirow{2}{*}{$\lambda$} & \multicolumn{3}{c}{\centering CIFAR100} \\
\cline{3-5}
& &  Acc\% $\uparrow$ &  ECE\% $\downarrow$ & NLL $\downarrow$ \\
\hline
\multirow{3}{*}{ResNet50} & 1 &  70.26 & 1.33 & 1.24 \\
& 0.1 &  71.33 & \textbf{1.08} & \textbf{1.03} \\
& 0.01 &  \textbf{71.51} & 2.31 & 1.31 \\
\bottomrule
\end{tabular}
\end{table}

\begin{table}[t]
\centering\small 
\caption{Ablation on where to apply the SGLD sampling.
}
\begin{tabular}{c|ccc}
\toprule
Position & Acc\% $\uparrow$ &  ECE\% $\downarrow$ & Second/Iter $\downarrow$ \\
\hline
Pixel space & 73.94 & 1.89 & 25.15 \\
Feature stage 1 & 73.21 & 2.09 & 10.31 \\
Feature stage 2 & 73.48 & 1.73 & 5.21 \\
Feature stage 3 & 73.05 & 2.17 & 2.37 \\
\bottomrule
\end{tabular}
\label{tab:ablation_hidden_layers}
\end{table}

\section{Conclusion}\label{sec:conclusion}
This paper has addressed the closed-world problem in softmax that causes neural network classifiers to produce overconfident predictions by introducing an energy-based objective to model the open-world uncertainty. 

\noindent\textbf{Acknowledgement.} \small{This study is supported by NTU NAP and Microsoft Research Lab Asia, and under the RIE2020 Industry Alignment Fund – Industry Collaboration Projects (IAF-ICP) Funding Initiative, as well as cash and in-kind contribution from the industry partner(s).}
{\small
\bibliographystyle{ieee_fullname}
\bibliography{egbib}

\begin{thebibliography}{10}\itemsep=-1pt

\bibitem{arbel2021generalized}
Michael Arbel, Liang Zhou, and Arthur Gretton.
\newblock Generalized energy based models.
\newblock In {\em ICLR}, 2021.

\bibitem{bendale2016towards}
Abhijit Bendale and Terrance~E Boult.
\newblock Towards open set deep networks.
\newblock In {\em CVPR}, 2016.

\bibitem{bengio2013}
Yoshua Bengio, Gr{\'{e}}goire Mesnil, Yann~N. Dauphin, and Salah Rifai.
\newblock Better mixing via deep representations.
\newblock {\em CoRR}, abs/1207.4404, 2012.

\bibitem{che2020your}
Tong Che, Ruixiang Zhang, Jascha Sohl-Dickstein, Hugo Larochelle, Liam Paull,
  Yuan Cao, and Yoshua Bengio.
\newblock Your gan is secretly an energy-based model and you should use
  discriminator driven latent sampling.
\newblock {\em arXiv preprint arXiv:2003.06060}, 2020.

\bibitem{deng2009imagenet}
Jia Deng, Wei Dong, Richard Socher, Li-Jia Li, Kai Li, and Li Fei-Fei.
\newblock Imagenet: A large-scale hierarchical image database.
\newblock In {\em 2009 IEEE conference on computer vision and pattern
  recognition}, pages 248--255. Ieee, 2009.

\bibitem{du2019implicit}
Yilun Du and Igor Mordatch.
\newblock Implicit generation and modeling with energy based models.
\newblock In {\em NeurIPS}, 2019.

\bibitem{gal2016dropout}
Yarin Gal and Zoubin Ghahramani.
\newblock Dropout as a bayesian approximation: Representing model uncertainty
  in deep learning.
\newblock In {\em international conference on machine learning}, pages
  1050--1059, 2016.

\bibitem{gao2021learning}
Ruiqi Gao, Yang Song, Ben Poole, Ying~Nian Wu, and Diederik~P Kingma.
\newblock Learning energy-based models by diffusion recovery likelihood.
\newblock In {\em ICLR}, 2021.

\bibitem{geyer1991markov}
Charles~J Geyer.
\newblock Markov chain monte carlo maximum likelihood.
\newblock 1991.

\bibitem{geyer1992practical}
Charles~J Geyer.
\newblock Practical markov chain monte carlo.
\newblock {\em Statistical science}, pages 473--483, 1992.

\bibitem{goodfellow2014generative}
Ian~J Goodfellow, Jean Pouget-Abadie, Mehdi Mirza, Bing Xu, David Warde-Farley,
  Sherjil Ozair, Aaron Courville, and Yoshua Bengio.
\newblock Generative adversarial networks.
\newblock In {\em NeurIPS}, 2014.

\bibitem{grathwohl2019your}
Will Grathwohl, Kuan-Chieh Wang, J{\"o}rn-Henrik Jacobsen, David Duvenaud,
  Mohammad Norouzi, and Kevin Swersky.
\newblock Your classifier is secretly an energy based model and you should
  treat it like one.
\newblock {\em arXiv preprint arXiv:1912.03263}, 2019.

\bibitem{grathwohl2021no}
Will~Sussman Grathwohl, Jacob~Jin Kelly, Milad Hashemi, Mohammad Norouzi, Kevin
  Swersky, and David Duvenaud.
\newblock No {\{}mcmc{\}} for me: Amortized sampling for fast and stable
  training of energy-based models.
\newblock In {\em ICLR}, 2021.

\bibitem{guo2017calibration}
Chuan Guo, Geoff Pleiss, Yu Sun, and Kilian~Q Weinberger.
\newblock On calibration of modern neural networks.
\newblock In {\em ICML}, 2017.

\bibitem{he2016deep}
Kaiming He, Xiangyu Zhang, Shaoqing Ren, and Jian Sun.
\newblock Deep residual learning for image recognition.
\newblock In {\em CVPR}, 2016.

\bibitem{hein2019relu}
Matthias Hein, Maksym Andriushchenko, and Julian Bitterwolf.
\newblock Why relu networks yield high-confidence predictions far away from the
  training data and how to mitigate the problem.
\newblock In {\em CVPR}, 2019.

\bibitem{hendrycks2019robustness}
Dan Hendrycks and Thomas Dietterich.
\newblock Benchmarking neural network robustness to common corruptions and
  perturbations.
\newblock {\em Proceedings of the International Conference on Learning
  Representations}, 2019.

\bibitem{hinton2002training}
Geoffrey~E Hinton.
\newblock Training products of experts by minimizing contrastive divergence.
\newblock {\em Neural computation}, 14(8):1771--1800, 2002.

\bibitem{hsu2020generalized}
Yen-Chang Hsu, Yilin Shen, Hongxia Jin, and Zsolt Kira.
\newblock Generalized odin: Detecting out-of-distribution image without
  learning from out-of-distribution data.
\newblock In {\em CVPR}, 2020.

\bibitem{TrustScore}
Heinrich Jiang, Been Kim, Melody Guan, and Maya Gupta.
\newblock To trust or not to trust a classifier.
\newblock In {\em Advances in Neural Information Processing Systems},
  volume~31, 2018.

\bibitem{kim2016deep}
Taesup Kim and Yoshua Bengio.
\newblock Deep directed generative models with energy-based probability
  estimation.
\newblock In {\em ICLR-W}, 2016.

\bibitem{krizhevsky2009learning}
Alex Krizhevsky, Geoffrey Hinton, et~al.
\newblock Learning multiple layers of features from tiny images.
\newblock 2009.

\bibitem{lakshminarayanan2017simple}
Balaji Lakshminarayanan, Alexander Pritzel, and Charles Blundell.
\newblock Simple and scalable predictive uncertainty estimation using deep
  ensembles.
\newblock In {\em NeurIPS}, 2017.

\bibitem{lecun1998gradient}
Yann LeCun, L{\'e}on Bottou, Yoshua Bengio, Patrick Haffner, et~al.
\newblock Gradient-based learning applied to document recognition.
\newblock {\em Proceedings of the IEEE}, 86(11):2278--2324, 1998.

\bibitem{lecun2006tutorial}
Yann LeCun, Sumit Chopra, Raia Hadsell, M Ranzato, and F Huang.
\newblock A tutorial on energy-based learning.
\newblock {\em Predicting structured data}, 1(0), 2006.

\bibitem{lee2018training}
Kimin Lee, Honglak Lee, Kibok Lee, and Jinwoo Shin.
\newblock Training confidence-calibrated classifiers for detecting
  out-of-distribution samples.
\newblock In {\em ICLR}, 2018.

\bibitem{liang2018enhancing}
Shiyu Liang, Yixuan Li, and Rayadurgam Srikant.
\newblock Enhancing the reliability of out-of-distribution image detection in
  neural networks.
\newblock In {\em ICLR}, 2018.

\bibitem{muller2019does}
Rafael M{\"u}ller, Simon Kornblith, and Geoffrey Hinton.
\newblock When does label smoothing help?
\newblock In {\em NeurIPS}, 2019.

\bibitem{naeini2015obtaining}
Mahdi~Pakdaman Naeini, Gregory Cooper, and Milos Hauskrecht.
\newblock Obtaining well calibrated probabilities using bayesian binning.
\newblock In {\em Proceedings of the AAAI Conference on Artificial
  Intelligence}, volume~29, 2015.

\bibitem{nalisnick2019do}
Eric Nalisnick, Akihiro Matsukawa, Yee~Whye Teh, Dilan Gorur, and Balaji
  Lakshminarayanan.
\newblock Do deep generative models know what they don't know?
\newblock In {\em ICLR}, 2019.

\bibitem{netzer2011reading}
Yuval Netzer, Tao Wang, Adam Coates, Alessandro Bissacco, Bo Wu, and Andrew~Y
  Ng.
\newblock Reading digits in natural images with unsupervised feature learning.
\newblock In {\em NeurIPS-W}, 2011.

\bibitem{NguyenYC14}
Anh~Mai Nguyen, Jason Yosinski, and Jeff Clune.
\newblock Deep neural networks are easily fooled: High confidence predictions
  for unrecognizable images.
\newblock {\em CoRR}, abs/1412.1897, 2014.

\bibitem{nijkamp2019learning}
Erik Nijkamp, Mitch Hill, Song-Chun Zhu, and Ying~Nian Wu.
\newblock Learning non-convergent non-persistent short-run mcmc toward
  energy-based model.
\newblock {\em arXiv preprint arXiv:1904.09770}, 2019.

\bibitem{padhy2020revisiting}
Shreyas Padhy, Zachary Nado, Jie Ren, Jeremiah Liu, Jasper Snoek, and Balaji
  Lakshminarayanan.
\newblock Revisiting one-vs-all classifiers for predictive uncertainty and
  out-of-distribution detection in neural networks.
\newblock {\em arXiv preprint arXiv:2007.05134}, 2020.

\bibitem{simonyan2014very}
Karen Simonyan and Andrew Zisserman.
\newblock Very deep convolutional networks for large-scale image recognition.
\newblock {\em arXiv preprint arXiv:1409.1556}, 2014.

\bibitem{szegedy2016rethinking}
Christian Szegedy, Vincent Vanhoucke, Sergey Ioffe, Jon Shlens, and Zbigniew
  Wojna.
\newblock Rethinking the inception architecture for computer vision.
\newblock In {\em Proceedings of the IEEE conference on computer vision and
  pattern recognition}, pages 2818--2826, 2016.

\bibitem{thulasidasan2019mixup}
Sunil Thulasidasan, Gopinath Chennupati, Jeff Bilmes, Tanmoy Bhattacharya, and
  Sarah Michalak.
\newblock On mixup training: Improved calibration and predictive uncertainty
  for deep neural networks.
\newblock {\em arXiv preprint arXiv:1905.11001}, 2019.

\bibitem{wald2021calibration}
Yoav Wald, Amir Feder, Daniel Greenfeld, and Uri Shalit.
\newblock On calibration and out-of-domain generalization.
\newblock {\em arXiv preprint arXiv:2102.10395}, 2021.

\bibitem{welling2011bayesian}
Max Welling and Yee~W Teh.
\newblock Bayesian learning via stochastic gradient langevin dynamics.
\newblock In {\em ICML}, 2011.

\bibitem{Xing2020Distance-Based}
Chen Xing, Sercan Arik, Zizhao Zhang, and Tomas Pfister.
\newblock Distance-based learning from errors for confidence calibration.
\newblock In {\em ICLR}, 2020.

\bibitem{zhao2017energy}
Junbo Zhao, Michael Mathieu, and Yann LeCun.
\newblock Energy-based generative adversarial network.
\newblock In {\em ICLR}, 2017.

\bibitem{zhou2021domain}
Kaiyang Zhou, Ziwei Liu, Yu Qiao, Tao Xiang, and Chen~Change Loy.
\newblock Domain generalization in vision: A survey.
\newblock {\em arXiv preprint arXiv:2103.02503}, 2021.

\end{thebibliography}
}

\end{document}